\def\year{2020}\relax
\documentclass[letterpaper]{article} 
\usepackage{aaai20}  
\usepackage{times}  
\usepackage{helvet} 
\usepackage{courier}  
\usepackage[hyphens]{url}  
\usepackage{graphicx} 

\usepackage{graphicx}  
\usepackage{amsmath}
\usepackage{amssymb}
\usepackage{amsthm}
\usepackage{booktabs}
\usepackage{enumerate}
\usepackage{graphicx}
\usepackage{subcaption}
\usepackage{threeparttable}
\usepackage{bm}

\newtheorem{theorem}{Theorem}
\newtheorem{lemma}[theorem]{Lemma}
\newtheorem{definition}[]{Definition}

\title{Less Is Better: Unweighted Data Subsampling via Influence Function}
\author{Zifeng Wang\textsuperscript{\rm 1}, Hong Zhu\textsuperscript{\rm 2 \thanks{This author contributes equally with Zifeng Wang.}}, Zhenhua Dong\textsuperscript{\rm 2 \thanks{Corresponding author.}}, Xiuqiang He\textsuperscript{\rm 2}, Shao-Lun Huang\textsuperscript{\rm 1}\\ 
\textsuperscript{\rm 1}Tsinghua-Berkeley Shenzhen Institute, Tsinghua University
\textsuperscript{\rm 2}Noah's Ark Lab, Huawei\\
Email: wangzf18@mails.tsinghua.edu.cn, zhuhong8@huawei.com, dongzhenhua@huawei.com, \\
hexiuqiang1@huawei.com, shaolun.huang@sz.tsinghua.edu.cn 
}

\begin{document}
\maketitle

\begin{abstract}
In the time of \emph{Big Data}, training complex models on large-scale data sets is challenging, making it appealing to reduce data volume for saving computation resources by subsampling. Most previous works in subsampling are weighted methods designed to help the performance of subset-model approach the full-set-model, hence the weighted methods have no chance to acquire a subset-model that is better than the full-set-model. However, we question that \emph{how can we achieve better model with less data?} In this work, we propose a novel Unweighted Influence Data Subsampling (UIDS) method, and prove that the subset-model acquired through our method can outperform the full-set-model. Besides, we show that overly confident on a given test set for sampling is common in Influence-based subsampling methods, which can eventually cause our subset-model's failure in out-of-sample test. To mitigate it, we develop a probabilistic sampling scheme to control the \emph{worst-case risk} over all distributions close to the empirical distribution. The experiment results demonstrate our methods superiority over existed subsampling methods in diverse tasks,  such as text classification, image classification, click-through prediction, etc.
\end{abstract}

\section{Introduction}
Bigger data can probably train a better model. It is almost the common sense nowadays for machine learning and deep learning practitioners. In the classical \emph{Empirical Risk Minimization} (ERM) theory, it assumes that our training samples and test samples are i.i.d drawn from the same distribution. The learning theory tries to minimize the estimate of the generalization risk, namely the empirical risk. Therefore, when training samples are large enough, we can have hypothesis function working well on the test set by optimizing on the empirical risk.

However, the ERM faces several challenges: 1) a model is learned from training set which is generated by $P(x,y)=p_{train}(x)p(y|x)$, and the model is tested on $Q(x,y)=q_{test}(x)p(y|x)$. The \emph{distribution shift} from $P$ to $Q$ violates the ERM's basic assumption; 2) unknown noise in data and its label is common in reality, causing some examples harmful for model's performance \cite{SzegedyZSBEGF13,Zhang2019TheoreticallyPT}; 3) training on large data sets imposes significant burden on computation, some large-scale deep learning models require hundreds even thousands of GPUs.

Specifically, the \emph{subsampling} approaches are initially proposed to cope with the last challenge. By virtue of subtle sampling regime, the selected subset can best approximate the original full set in terms of data distribution, hence the model can be trained on a compressed version of data set.
In our work, we attempt to design sampling regime that can not only reduce the computation complexity, but also deal with several other ERM's difficulties. For instance, reweighting examples by sampling probabilities to fix mismatch between $P$ and $Q$, and dropping noisy samples to strengthen the model's generalization ability.

Our work can be outlined in four points. First, instead of approaching the full-set-model $\hat{\theta}$, we prove that the model  $\tilde{\theta}$ trained by a selected subset through our subsampling method can outperform the $\hat{\theta}$; second, we propose several probabilistic sampling functions, and analyze how the sampling function influences the worst-case risk \cite{Bagnell05} changes over a $\chi^2$-divergence ball. We further propose a surrogate metric to measure the confidence degree of the sampling methods over the observed distribution, which is useful for evaluating model's generalization ability on a set of distributions; third, for the sake of implementation efficiency, the Hessian-free mixed Preconditioned Conjugate Gradient (PCG) method is used to compute the influence function (IF) in sparse scenarios; last, complete experiments are conducted on diverse tasks to demonstrate our methods superiority over the existing state-of-the-art subsampling methods\footnote{The code can be found at \url{https://github.com/RyanWangZf/Influence_Subsampling}}.

\subsection{Related work}
There are two main ideas to cope with the ERM challenges aforementioned: 1) pessimistic method that tries to learn model robust to noise or bad examples, including $l_2$-norm regularization, Adaboost \cite{FreundS97}, hard example mining \cite{MalisiewiczGE11} and focal loss \cite{LinGGHD17}; and 2) optimistic method that modifies the input distribution directly. There are several genres of optimistic methods:
the example reweighting method is used for dealing with distribution shift by \cite{Bagnell05,Hu2016DoesDR}, and handling data bias by \cite{KumarPK10,ren2018learning}; the sample selection method is applied to inspect and fix mislabeled data by \cite{Zhang0W18}. However, few of them have worked on alleviating the computational burden in terms of big data.

In order to reduce computation, the weighted subsampling method has been explored to approximate the maximum likelihood with a subset on the logistic regression \cite{fithian2014local,wang2018optimal}, and on the generalized linear models \cite{ai2018optimal}. \cite{ting2018optimal} introduces the IF in weighted subsampling for asymptotically optimal sampling probabilities for several generalized linear models. However, it is still an open problem about how to treat high variance of weight terms for weighted subsampling.

Specifically, the IF is defined by the Gateaux derivatives within the scope of \emph{Robust Statistics} \cite{huber2011robust}, and extended to measure example-wise influence \cite{koh2017understanding} and feature-wise influence \cite{SliwinskiSZ19} on validation loss. The family of IF is mainly applied to design adversarial example and explain behaviour of black-box model previously. Recently, the IF on validation loss is used for targeting important samples, \cite{wang2018data} builds a sample selection scheme on Deep Convolutional Networks (CNN), and \cite{sharchilev2018finding} builds specific influential samples selection algorithm for Gradient Boosted Decision Trees (GBDT). However, by far there has no systematic theorem to guide IF's use in subsampling. Our work tries to build theoretical guidance for IF-based subsampling, which combines reweighting and subsampling together to synthetically cope with ERM's challenges, e.g. distribution shift and noisy data.

\section{Preliminaries}
\label{theory}
Training samples $z_1,z_2,..,z_n \in \mathcal{X} \times \mathcal{Y} $ are generated from $P(x,y)$, and $x\in\mathbb{R}^d$ where the $d$ is the number of feature dimension. Specifically for classification task, we have hypothesis function $h_\theta \in\mathcal{H}: \mathbb{R}^d \to \mathbb{R}$ parameterized by $\theta \in \Theta$. The goal is to minimize the 0-1 \emph{risk} $\mathcal{R}_\theta(P)=\mathbb{E}_{P}[\mathbb{I}(h_\theta(x) \neq y)]$, and learn the optimal $\hat{\theta}$. For computational tractability, researchers focus on minimization of the \emph{surrogate loss}, e.g. the log loss for binary classification:
\begin{equation}
    l(z,h_\theta)=-y\log h_\theta(x)-(1-y)\log(1-h_\theta(x))
\end{equation}
Therefore, the risk minimization problem can be empirically approximated by $\hat{\theta} \triangleq \arg\min_{\theta} \frac1n\sum_{i=1}^n l(z_i,h_\theta)$. For convenience, we denote $l(z_i,h_\theta)$ by $l_i(\theta)$. The main notations are listed in Table \ref{notation}.

\begin{table}[htbp]
  \centering
  \caption{Main notation.}
    \begin{tabular}{l|l}
    $z_i$, $z_j$     & training and testing sample, $\in\mathbb{R}^d\times\mathbb{R}$. \\
    $\hat{\theta}$, $\tilde{\theta}$(or $\hat{\theta}_\epsilon$)     &  full-set-model and subset-model, $\in\mathbb{R}^d$.\\
    $\phi_i(\hat{\theta})$     & $z_i$'s influence of whole test set risk, $\in\mathbb{R}$. \\
    $\psi_\theta(z_i)$     & $z_i$'s influence of model parameter, $\in\mathbb{R}^d$. \\
    $\epsilon_i$     & perturbation put on $z_i$'s loss term, $\in\mathbb{R}$. \\
    $\pi_i$     & sampling probability of $z_i$, $\in\mathbb{R}^+$. \\
    $l_i(\theta)$     & model $\theta$'s risk on training sample $z_i$, $\in\mathbb{R}$.\\
    $l_j(\theta)$     & model $\theta$'s risk on test sample $z_j$, $\in\mathbb{R}$.\\
    $P$ & training distribution. \\
    $Q^\prime$ & a specific test distribution. 
    \end{tabular}%
  \label{notation}%
\end{table}%

\subsubsection{Weighted subsampling.}
For a general subsampling framework, each sample is assigned with a random variable $O_i \sim Bern(\pi_i)$, indicating whether this sample is selected or not, such that [$O_i=1]$ $\leftrightarrow$ [$z_i$ is selected]. The weighted subsampling methods have the similar form of objective function $\mathcal{R}_w$ on the subset:
\begin{equation}
\mathcal{R}_w=\frac1n\sum_{i,O_i=1}^n\frac1{\pi_i}l_i(\theta)
\end{equation}
where each term is assigned the inversion of its sampling probability $1/\pi_i$ as the weight. It is similar to the technique used in \emph{Causal Inference} to handle the selection bias \cite{schnabel2016recommendations}, the Eq.\eqref{expectlw} derives the expectation of $\mathcal{R}_w$ on $O$:
\begin{equation} \label{expectlw}
    \mathbb{E}_O(\mathcal{R}_w)=\frac1n\sum_{i=1}^n\mathbb{E}_{O_i}(O_i\times\frac1{\pi_i}l_i(\theta))=\frac1n\sum_{i=1}^n l_i(\theta)
\end{equation}
The expectation of the $\mathcal{R}_w$ on a subset is the same as the empirical risk on the full set, which means the weighted subsampling methods aim at finding optimal $\{\pi_i\}_{i=1}^n$ to let the subset risk minimizer $\tilde{\theta}$ as close to the full-set risk minimizer $\hat{\theta}$ as possible.

This weighted approach has three main challenges: 1) we need to modify the existing training procedures to accommodate the weighted loss function; 2) because $\pi$ can be small and its inversion ranges widely, the weighted loss function suffers from high variance; and 3) most importantly, as the weighted methods build a consistent estimator of the full-set-model \cite{ai2018optimal}, it theoretically assumes that subset-model cannot outperform the full-set-model.

\subsubsection{Unweighted subsampling.}
We propose a novel \emph{Unweighted} subsampling method which does not require $\frac1{\pi_i}$ in its objective function:
\begin{equation} \label{uwrisk}
    \mathcal{R}_{uw}=\frac1{|\{i,O_i=1\}|}\sum_{i,O_i=1}l_i(\theta)
\end{equation}
where the $|\cdot|$ means cardinality of the set. From Eq.\eqref{uwrisk}, we can find the expectation of subset-model's risk is no longer equal to the full-set-model's.
This formula can be seen as reweighting the samples with respect to their sampling probabilities implicitly. It directly overcomes the first two challenges in weighted subsampling above-mentioned, but further efforts are required to solve the last one, which will be introduced in the following section. An intuitive demonstration of the difference between the weighted and unweighted methods is shown as Fig. \ref{intuitive}.

\begin{figure}[t]
\centering
\includegraphics[width=8.5cm]{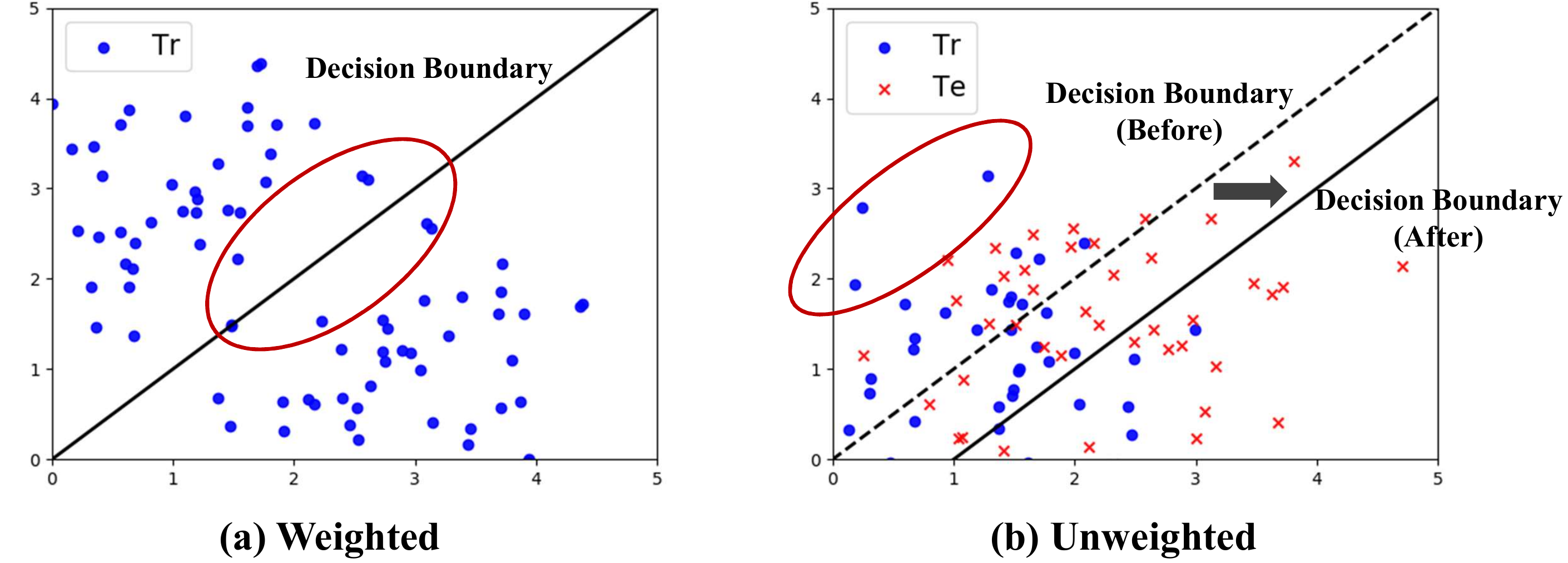}
\caption{(a) if the blue points (training samples) within the red circle are removed, the new optimal decision boundary is still same as the former one; (b) if removing blue points in the red circle, the new decision boundary shifts from the left, while achieves better performance on the Te set.}
\label{intuitive}
\end{figure}

\subsubsection{Influence functions.}
If the $i^\prime$th training sample is upweighted by a small $\epsilon^\prime$ corresponding to its loss term, then the perturbed risk minimizer is
\begin{equation}
    \hat{\theta}_{\epsilon^\prime} \triangleq \arg \min_{\theta\in\Theta} \frac1n \sum_{i=1}^n l_i(\theta)+\epsilon^\prime l_{i^\prime}(\theta)
\end{equation}

The basic idea of IF is to approximate the change of parameter $\theta$ \cite{cook1980characterizations}:
\begin{equation}\label{phi}
        \psi_{\theta}(z_i)  \triangleq \left. \frac{d\hat{\theta}_{\epsilon}}{d\epsilon} \right|_{\epsilon=0} = -H^{-1}_{\hat{\theta}}\nabla_{\theta} l_i(\hat{\theta}) 
\end{equation}
or test risk on a given test distribution $Q^\prime$ from \cite{koh2017understanding}:
\begin{equation}\label{phi_loss}
    \begin{split}
        \phi(z_i\sim P, z_j\sim Q^\prime)  &\triangleq \left. \frac{dl_j(\hat{\theta}_{\epsilon})}{d\epsilon}\right |_{\epsilon=0} \\
&=-\nabla_{\theta}l_j(\hat{\theta})^{\top}H^{-1}_{\hat{\theta}}\nabla_{\theta}l_i(\hat{\theta})
    \end{split}
\end{equation}
where we have $z_i$, $z_j$ from training and test set, and the $H_{\hat{\theta}}\triangleq\frac1n\sum_{i=1}^n \nabla^2_{\theta}l_i(\hat{\theta})$ is the Hessian matrix based on full set risk minimizer $\hat{\theta}$, which is positive definite (PD) if the empirical risk is twice-differentiable and strictly convex in $\theta$.

\section{Methodology}
The key challenge lies on that the $\hat{\theta}$ may not be the best risk minimizer corresponding to $Q^{\prime}$, due to the distribution shift between $P$ and $Q^\prime$ and unknown noisy samples in the training set. With the advent of IF, we can measure one sample's influence on test distribution without prohibitive leave-one-out training. The essential philosophy here is, given a test distribution $Q^\prime$, some samples in training set cause increasing test risk. If they are downweighted, accordingly we can have less test risk than before, namely $\mathcal{R}_{\tilde{\theta}}(Q^\prime) \leq \mathcal{R}_{\hat{\theta}}(Q^\prime)$ where the $\tilde{\theta}$ is the new model learned after some harmful samples are downweighted.

\subsection{Subset can be better}

Considering there are $n$ perturbations put on each training sample: $\vec{\epsilon}=(\epsilon_1,\epsilon_2,...,\epsilon_n)^{\top}$, the perturbed risk minimizer is denoted as $\hat{\theta}_\epsilon$.
Given $m$ samples $\{z_j\}_{j=1}^m$ from another distribution $Q^{\prime}$, the objective is to design the $\vec{\epsilon}$ that minimizes the test risk $\mathcal{R}_{\hat{\theta}_\epsilon}(Q^\prime)=\frac1m \sum_{j=1}^ml_j(\hat{\theta}_{\epsilon})$.

According to the definition of IF in Eq. \eqref{phi_loss}, we can approximate the loss change of $z_j\sim Q^\prime$ if $z_i\sim P$ is upweighed by a $\epsilon_i$:
\begin{equation}
    l_j(\hat{\theta}_{\epsilon_i})-l_j(\hat{\theta})\approx\epsilon_i \times \phi(z_i,z_j,\hat{\theta})
\end{equation}
which can be extended to the whole test distribution as following:
\begin{equation}
\begin{split}
   \mathcal{R}_{\hat{\theta}_{\epsilon_i}}(Q^\prime)-\mathcal{R}_{\hat{\theta}}(Q^\prime) &= \frac1m \sum_{j=1}^m(l_j(\hat{\theta}_{\epsilon_i})-l_j(\hat{\theta})) \\
   & \approx \epsilon_i \times \frac1m\sum_{j=1}^m\phi(z_i,z_j,\hat{\theta})
\end{split}
\end{equation}

For convenience, we use $\phi_i(\hat{\theta})\triangleq \sum_{j=1}^m\phi(z_i,z_j,\hat{\theta})$ to indicate one training sample $z_i$'s influence over whole $Q^\prime$. Therefore, with the $\hat{\theta}_\epsilon$ under the perturbation of $\vec{\epsilon}$, the test risk change can be approximated as following\footnote{We assume that all elements in $\vec{\epsilon}$ are small, and each training sample $z_i$ influences the test risk independently, hence we add up all terms linearly for simplicity of implementation.}:
\begin{equation} \label{testriskminimize}
\begin{split}
    \mathcal{R}_{\hat{\theta}_{\epsilon}}(Q^\prime)-\mathcal{R}_{\hat{\theta}}(Q^\prime) &= \frac1m\sum_{i=1}^n\sum_{j=1}^m (l_j(\hat{\theta}_{\epsilon_i})-l_j(\hat{\theta}))  \\
    & \approx \frac1m \sum_{i=1}^n\sum_{j=1}^m\epsilon_i\phi(z_i,z_j,\hat{\theta}) \\
    &= \frac1m \sum_{i=1}^n\epsilon_i\phi_i(\hat{\theta})
\end{split}
\end{equation}

Specifically, suppose when $\pi_i=1$ for $\forall i=1,2,...,n$, the subset is the same as the full set. In this situation, the $\hat{\theta}_{\epsilon}=\hat{\theta}$ such that $l_j(\hat{\theta}_\epsilon)-l_j(\theta)=0$ for $\forall j=1,2..,m$. Based on this analysis, we have the Lemma \ref{avgphi}. For clear notion, we use bold letters to represent random variables, such that the $\phi_i(\hat{\theta})$ and $\epsilon_i$ are the realization of the random variables $\bm{\phi}$ and $\bm{\epsilon}$, respectively.

\begin{lemma}\label{avgphi}
The expectation of the influence function $\phi_i(\hat{\theta})$ over training distribution $P$ is always 0, which means:
\begin{equation}
    \mathbb{E}_{P}(\bm{\phi})\approx\frac1n\sum_{i=1}^n\phi_i(\hat{\theta})=0
\end{equation}
\end{lemma}

According to the Eq. \eqref{testriskminimize}, minimizing test risk is equivalent to minimizing the objective function $\frac1n\sum_{i=1}^n\epsilon_i\phi_i(\hat{\theta})$. Actually, this objective function is empirical form of the $\mathbb{E}_{P}(\bm{\phi} \times \bm{\epsilon})$ from which we derive the Lemma \ref{nonneg}.

\begin{lemma}\label{nonneg}
The subset-model $\hat{\theta}_\epsilon$
performs not worse than the full-set-model $\hat{\theta}$ in terms of test risk $\mathcal{R}(Q^\prime)$ if $\bm{\epsilon}$ and $\bm{\phi}$ are negative correlated:
\begin{equation}
  Cov(\bm{\phi},\bm{\epsilon}) \leq 0
\end{equation}
\end{lemma} 

The Lemma \ref{nonneg} gives instruction that making the perturbation $\bm{\epsilon}$ negative correlated with IF $\bm{\phi}$ can ensure better subset-model $\hat{\theta}_\epsilon$. To this end we can let the $\epsilon(\phi)$ a decreasing function. The proof of Lemma \ref{avgphi} and Lemma \ref{nonneg} can be seen in the supplementary appendix A and B, respectively.

\subsection{Deterministic v.s. Probabilistic sampling}
Similar to the Eq.\eqref{expectlw}, the expectation of $\phi$ on subset via the observation variable $O$ can be acquired by:
\begin{equation}\label{expphi}
    \mathbb{E}_{O}(\sum_{i,O_i=1}^n\phi_i(\hat{\theta}))=\sum_{i=1}^n\mathbb{E}_{O_i}(O_i\times\phi_i(\hat{\theta}))=\sum_{i=1}^n\pi_i\phi_i(\hat{\theta})
\end{equation}

However, the objective function in Eq. \eqref{testriskminimize} is defined on perturbation $\epsilon$ instead of sampling probability $\pi$, so that we need to bridge the gap between them:
\begin{equation}\label{piandepsilon}
\begin{split}
\{\epsilon_i^{\star}\}_{i=1}^n = \arg\min_{\vec{\epsilon}} f(\vec{\epsilon}) &= \arg\min_{\vec{\epsilon}} n(f(\vec{\epsilon})+\frac1n \sum_{i=1}^n \phi_i(\hat{\theta})) \\
    &= \arg\min_{\vec{\epsilon}} \sum_{i=1}^n(n\epsilon_i+1)\phi_i(\hat{\theta})
\end{split}
\end{equation}
The Eq. \eqref{piandepsilon} holds because from Lemma \ref{avgphi} we know the $\mathbb{E}_{P}(\bm{\phi})=0$. Here we assume that $\epsilon\in [-\frac1n,0]$ \footnote{The $\epsilon=0$ means no perturbation is applied, while the $\epsilon=-\frac1n$ means a sample is totally dropped in objective function, here all perturbations are assumed within this interval.}, therefore, if we let $\pi = n\epsilon+1 \in [0,1]$, then the perturbation is transformed to sampling probability, because the Eq.\eqref{expphi} and Eq. \eqref{piandepsilon} are in the same form.

In fact, the Eq.\eqref{expphi} has closed form of optimal $\pi_i$:
\begin{equation} \label{dropoutsamplingfunc}
\pi_i^{\star}= \left\{\begin{array}{rcl}
1 & & {\phi_i \leq 0} \\
0 & & {\phi_i > 0}
\end{array}
\right.
\end{equation}
This form of sampling is termed as \emph{Data dropout} in \cite{wang2018data} while here we call it \emph{Deterministic sampling}, because it simply sets a threshold and selects samples deterministically.  By contrast, sampling with a continuous function $\pi(\phi) \in [0,1]$ is called \emph{Probabilistic sampling} since each sample has a probability to be selected. Most of sampling studies belong to probabilistic sampling methods: \cite{wang2018optimal} builds $\pi^{mMSE}$ and $\pi^{mVc}$ based on \emph{A-optimality} and \emph{L-optimality} respectively, and \cite{ting2018optimal} uses $\pi_i\propto \|\psi_{\theta}(z_i)\|$. 

\subsection{Analysis of sampling functions}
For Dropout method, the sample's influence over a $Q^\prime$ is the essential criterion for sampling, such that the obtained subset-model ends up being optimal only for a $Q^\prime$. However, the subset-model's robustness to distribution shift \cite{HuNSS18} is also a concern. That is, for a set of distributions around the empirical one, whether our subset-model can still maintain its performance. In this work, we postulate that the Influence-based subsampling confronts to trade the subset-model's performance on a specific $Q^\prime$ off its distributional robustness. In this viewpoint, the Dropout method is \emph{overly confident} on a $Q^\prime$ at the expense of deteriorating generalization ability, hence it is reasonable to measure and control this \emph{confidence degree} for our subsampling methods. 

Considering an uncertainty set $\mathcal{Q}=\{Q \mid Q \ll P, D_{\chi^2}(Q\|P)\leq \delta, \delta \geq 0\}$, where the $Q \ll P$ denotes that $Q$ is absolutely continuous w.r.t. $P$, and $D_{\chi^2}(\cdot \| \cdot)$ means $\chi^2$-divergence:
\begin{equation}
    D_{\chi^2}(Q\|P) = \mathbb{E}_P[\frac12(1-\frac{dQ}{dP})^2]
\end{equation}
The $\mathcal{Q}$ is a $\chi^2$-divergence distribution ball, indicating all neighborhoods of the empirical distribution $P$. The worst-case risk $\mathcal{R}_{\hat{\theta}_\epsilon}(Q)$ is defined as the supremum of the risk over any $Q\in\mathcal{Q}$ \cite{Bagnell05}:
\begin{equation} \label{suprisk}
    \mathcal{R}_{\hat{\theta}_\epsilon}(Q) \triangleq \sup_{Q\in\mathcal{Q}} \{ \mathbb{E}_Q [l(\hat{\theta}_\epsilon;Z)]\}
\end{equation}
From the \cite{Duchi2018LearningMW}, the dual form of the Eq. \eqref{suprisk} is:
\begin{equation} \label{infrisk}
    \mathcal{R}_{\hat{\theta}_\epsilon}(\eta) = \inf_{\eta \in \mathbb{R}} \left \{ \sqrt{2\delta+1}\times \mathbb{E}_P[(l(\hat{\theta}_\epsilon;Z)-\eta)_{+}^2]^{\frac12} + \eta \right \}
\end{equation}
where the $\eta$ is the dual variable. This duality transforms the supremum in Eq.\eqref{suprisk} to a convex function on the empirical distribution $P$, thus allowing us to measure the worst-case risk quantitatively. Before focusing on analyzing how the worst-case risk $\mathcal{R}_{\hat{\theta}_\epsilon}(Q)$ changes with the different sampling functions in Theorem \ref{liprisk}, we need to introduce two terms:

\begin{definition}
A function f(x): $\mathbb{R}^d\to\mathbb{R}$ is said to be Lipschitz continuous with constant $\xi$ if
$\|\nabla f(x)-\nabla f(y)\| \leq \xi \|x-y\|, \forall x,y \in \mathbb{R}^d$.
\end{definition}

\begin{definition}
A function f(x) has $\sigma$-bounded gradients if $\|\nabla f(x)\|\leq\sigma$ for all $x\in \mathbb{R}^d$.
\end{definition}

\begin{theorem}\label{liprisk}
Let $\eta^{*}$ the optimal dual variable $\eta$ that achieves the infimum in the Eq. \eqref{infrisk}, and the perturbation function $\epsilon(\phi)$ has $\sigma$-bounded gradients. Then, the worst-case risk $\mathcal{R}_{\hat{\theta}_\epsilon}(\eta^{*})$ is a Lipschitz continuos function w.r.t. the IF vector $\vec{\phi}=(\phi_1,\phi_2,..,\phi_n)^{\top}$ where we have the Lipschitz constant $\xi=\mathcal{O}(\sigma\frac{\sqrt{2\delta+1}}{n})$, that is
\begin{equation*}
    \|\nabla_{\vec{\phi}}\mathcal{R}_{\hat{\theta}_\epsilon}(\eta^*)\| \leq \sigma\frac{\sqrt{2\delta+1}}{n} \times \sqrt{\sum_{i=1}^n \phi_i^2}
\end{equation*}
\end{theorem}

The Theorem \ref{liprisk} relate the change rate of the worst-case risk to the gradient bound $\sigma$ of the perturbation function $\epsilon(\phi)$. For the Dropout method, its sampling function Eq.\eqref{dropoutsamplingfunc} has unbounded gradient since it is inconsistent at the zero point, causing the $\sigma \to \infty$. This property makes $\mathcal{R}_{\hat{\theta}_\epsilon}(\eta^*)$ no longer Lipschitz continuous and suffers sharp fluctuation. By contrast, our probabilistic methods can adjust the confidence degree by tuning the $\sigma$. This is crucial to avoid over confidence on a specific $Q^\prime$ that leads to large risk on other $Q \in \mathcal{Q}$. In fact, our experiments bear out that our probabilistic methods maintain its performance out-of-sample with proper $\sigma$, while the Dropout method often crashes. The proof of Theorem \ref{liprisk} can be found in Appendix C.

\subsection{Surrogate metric for confidence degree}
Nevertheless, we find that $\sigma$ is the determinant of confidence degree, it is still intractable to measure this degree quantitatively, which is important to guide our methods use in practice. Empirically, to deal with over fitting, practitioners prefer adding constraints on the model's parameters $\theta$, e.g. $l_2$-norm regularizer. In our theory, we propose to apply the $\Gamma(\vec{\phi})=\|\hat{\theta}_\epsilon - \hat{\theta}\|^2$ to evaluate the confidence degree over a specific $Q^\prime$. We term the $\Gamma(\vec{\phi})$ a \emph{surrogate metric} for confidence degree, and prove in Theorem \ref{lip} that it is reasonable because the $\Gamma(\vec{\phi})$ has the same magnitude of Lipschitz constant as the $\mathcal{R}_{\hat{\theta}_\epsilon}(\eta^*)$. In detail, the worst-case risk and our surrogate metric share the same change rate corresponding to the sampling function's gradient bound $\sigma$.

\begin{theorem}\label{lip}
Let the perturbation function $\epsilon(\phi)$ has $\sigma$-bounded gradient, and the $|\epsilon(\phi)|$ is bounded by $\tau\in\mathbb{R}^+$, that is $|\epsilon(\phi)|\leq\tau$. We have the parameter shift $\Gamma(\vec{\phi})=\|\hat{\theta}_\epsilon -\hat{\theta}\|^2$ is Lischitz continuous with its Lipschitz constant $\xi=\mathcal{O}(\sigma\tau)$. Specifically for $\tau=\frac1n$, we have $\xi = \mathcal{O}(\frac{\sigma}n)$.
\end{theorem}

The Theorem \ref{lip} is helpful to measure our sampling methods confidence degree in practice if the radius $\delta$ of the $\chi^2$-divergence ball is unknown. Theoretically, a relatively small $\sigma$ guarantees more robust model. The proof of the Theorem \ref{lip} can be found in the supplementary appendix D. 

\begin{figure}[htbp]
\centering
\includegraphics[width=6.5cm]{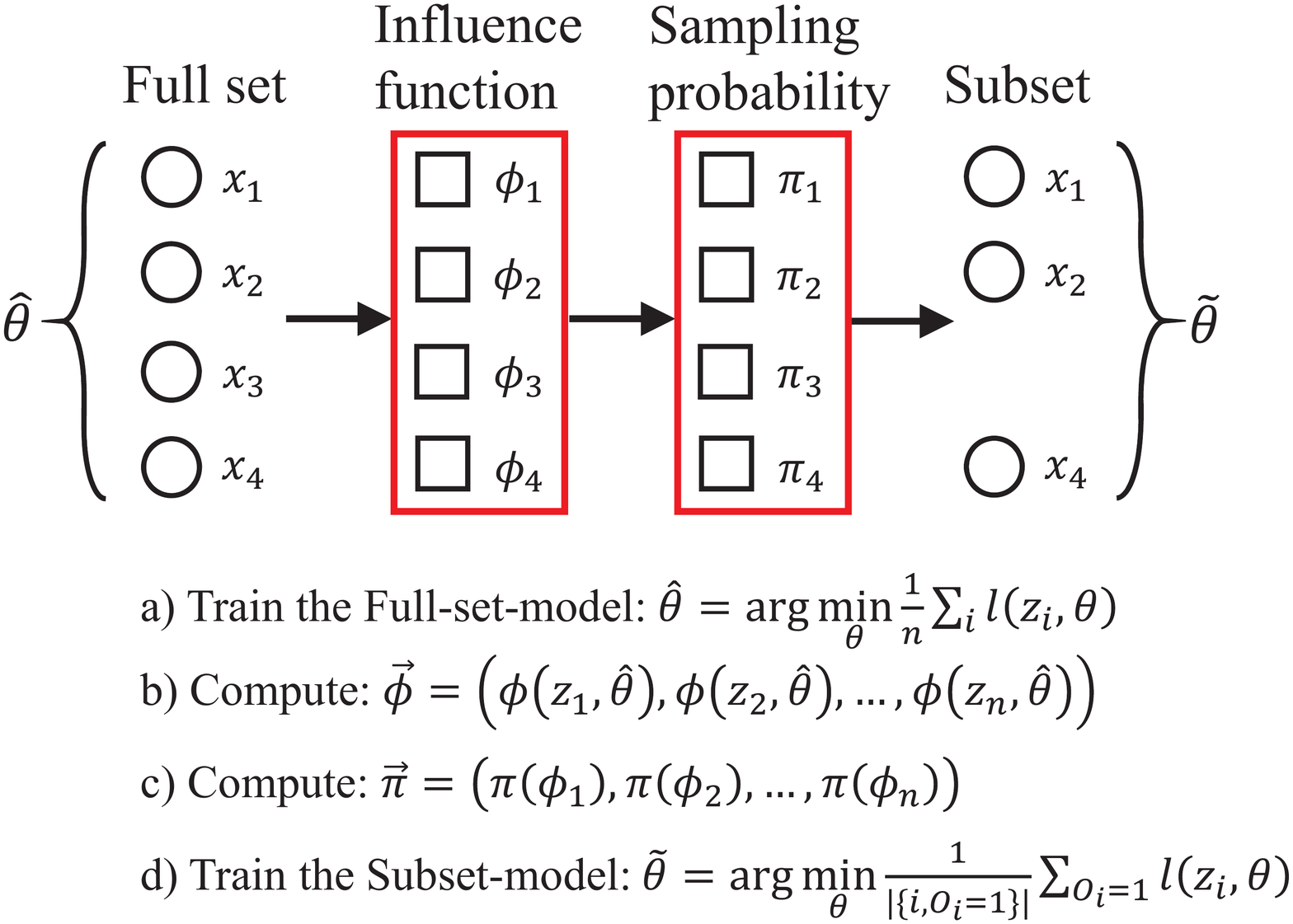}
\caption{Our unweighted subsampling framework.}
\label{framework}
\end{figure}

\section{Implementation}
In this section, the unweighted subsampling method is incorprated in our framework, shown in Fig.\ref{framework}: we train $\hat{\theta}$ with the full set data, and calculate the IF vector $\vec{\phi}=(\phi_1,\phi_2,...,\phi_n)$ on training set with the $\hat{\theta}$. The sampling probabilities are acquired with the designed probabilistic sampling function $\pi(\cdot)$ afterwards. We will discuss the two basic modules of this framework: 1) calculating the IF and 2) designing probabilistic sampling functions.

\subsection{Calculating influence functions}
The IF in Eq.\eqref{phi_loss} can be calculated in two steps: first calculating the \emph{inverse Hessian-vector-product} (HVP) $\nabla_{\theta}l_j(\hat{\theta})^{\top}H^{-1}_{\hat{\theta}}$, then multiply it with the $\nabla_{\theta}l_i(\hat{\theta})$ for each training sample. To handle the sparse scenarios when the $H_{\hat{\theta}}$ has high dimensions, \cite{martens2010deep} proposes to transform the inverse HVP into an optimization problem: $H^{-1}_{\hat{\theta}}v \equiv \arg \min_{t}\{\frac12 t^{\top} H_{\hat{\theta}}t-v^{\top}t\}$, and solve it with Newton conjugate gradient (Newton-CG) method. Moreover, the \cite{agarwal2017second} proves that the stochastic estimation makes the calculation feasible when the loss function $l(\cdot)$ is non-convex in $\theta$. These works ensure our framework's feasibility in both convex and non-convex scenarios. Without loss of generality, we mainly focus on convex scenarios.

When the CG converges slowly because of ill-conditioned sub-problem, the mixed preconditioner $\bar{M}$ is useful to reduce CG steps\cite{nash1985preconditioning,hsia2018preconditioned}:

\begin{equation}
    \bar{M}=\alpha\times diag(H_\theta)+(1-\alpha)\times I
\end{equation}
where the $\alpha$ is a weight parameter, $I$ is the identical matrix and $H_\theta$ is the Hessian matrix. Specifically for logistic regression model, its diagonal elements on $H_\theta$ is:
\begin{equation}
    (H_\theta)_{kk}=1+C\sum_{i=1}^n(\hat{y}_i-y_i)x_{ik}^2
\end{equation}
where the $C$ is the regularization parameter. Our experiments demonstrate that the Mixed PCG is efficacious for speeding up the calculation of IF.

\begin{figure}[t]
\centering
\includegraphics[width=9.0cm]{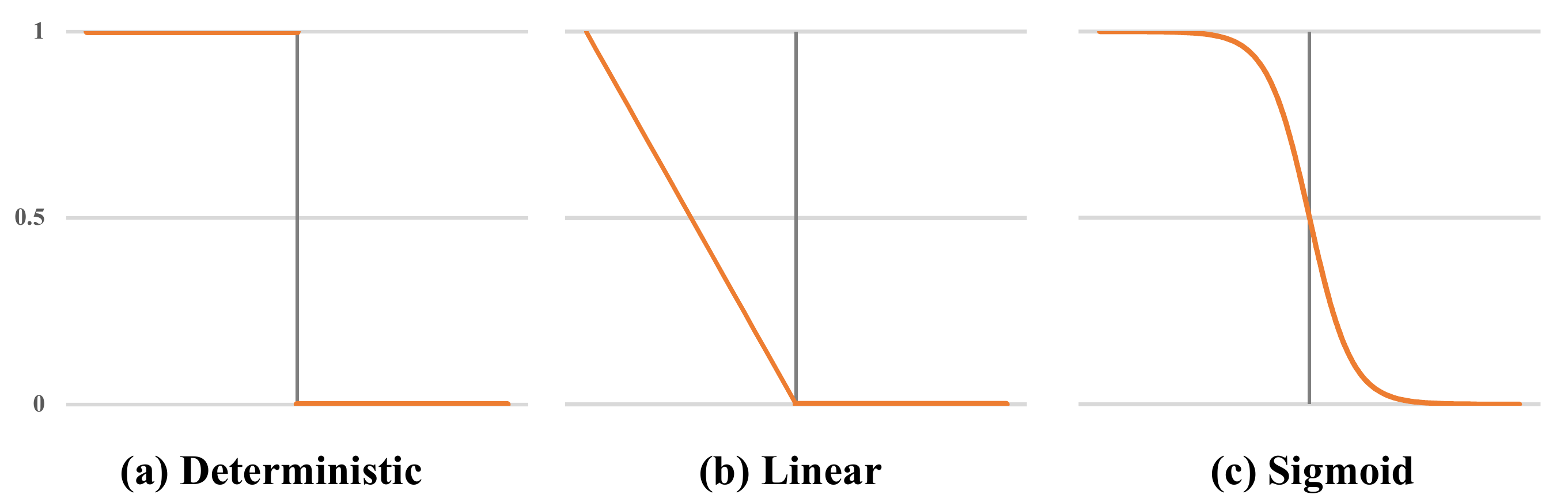}
\caption{A family of sampling functions, the $x$ axis is value of influence function $\phi$ and the $y$ axis is the probability $\pi$.}
\label{fig:sampling}
\end{figure}

\subsection{Probabilistic sampling functions}\label{samplingfunctions}
From Lemma \ref{nonneg}, better subset-model is ensured with a decreasing function $\pi(\phi)$. Furthermore, the Theorem \ref{liprisk} and \ref{lip} prove that the gradient bound $\sigma$ of $\pi(\phi)$ can be adjusted for confidence degree over a $Q^{\prime}$. We can design a family of probabilistic sampling functions with tunable hyper parameter w.r.t. $\sigma$. Here we develop two basic functions, termed as \emph{Linear sampling} and \emph{Sigmoid sampling}.

\subsubsection{Linear sampling.}
Inspired by \cite{ting2018optimal}, which builds $\pi_i \propto \| \psi_\theta(z_i) \|$ as $\pi_i = \max\{\alpha,\min\{1,\lambda\|\psi_\theta(z_i)\|\} \}$, we design a Linear sampling function where we let $\pi_i \propto -\phi_i(\hat{\theta})$:
\begin{equation}
    \pi_i=n\epsilon(\phi_i)+1=\max\{0,\min\{1,-\alpha\phi_i\}\}
\end{equation}
where $\alpha \in \mathbb{R}^+$. It is easy to prove the gradient bound $\sigma$ of $\pi(\phi)$ is $\alpha$, thus the degree of confidence relies on the $\alpha$ for the Linear sampling. It is a little different from the $\pi_i\propto \|\psi_\theta(z_i)\|$, because the $\phi_i$ can be both negative and positive, which means many samples can have zero probability to be sampled. If we set a relatively high sampling ratio, like $90\%$ or higher, we will never get enough samples in our subset. Empirically, we find randomly picking up the negative $\pi$ samples reaches relatively good results.

\subsubsection{Sigmoid sampling.}
The Sigmoid function is generally used in logistic regression to scale the outputs into $[0,1]$, which indicates probability of each class, such that here we can use it to transform $\phi_i$ to probability as following:
\begin{equation}
    \pi_i=n\epsilon(\phi_i)+1=\frac1{1+e^{\frac{\alpha\phi_i}{\max(\{\phi_i\})-\min(\{\phi_i\})}}}
\end{equation}
where $\alpha\in\mathbb{R}^+$. For the Sigmoid function, we can still adjust the $\alpha$ to make the probability distribution more flat or steep, thereby control the confidence degree.

\section{Experiments}
In this section, we present data sets and experiment settings at first, and introduce several baselines for comparison. After that, we do experiments to evaluate our methods in terms of effectiveness, robustness and efficiency.

\subsection{Data sets}
We perform extensive experiments on various public data sets which conclude many domains, including computer vision, natural language processing, click-through rate prediction, etc. Additionally, we test the methods on the \textbf{Company} data set which contains around \textbf{100 million} samples with over \textbf{10 million} features. They are queries collected from a real world recommender system, whose feature set contains user history behaviour, item's side information and contextual information, such as time and location. These data sets range from small to large, from low dimensions to high dimensions, which can testify the methods effectiveness and robustness in diverse scenarios. The data set statistics and more details about preprocessing on some data sets are described in appendix E.

\begin{table*}[htbp]
  \centering
  \begin{threeparttable}
  \caption{Average logloss evaluated on the out-of-sample Te set when sampling ratio is set to $95\%$.}
 \setlength{\tabcolsep}{4mm}{
    \begin{tabular}{lrrrrrr}
    \toprule
          & \multicolumn{1}{l}{Full set} & \multicolumn{1}{l}{Random} & \multicolumn{1}{l}{OptLR} & \multicolumn{1}{l}{Dropout} & \multicolumn{1}{l}{Lin-UIDS\tnote{*}} & \multicolumn{1}{l}{Sig-UIDS\tnote{*}} \\
    \midrule
    UCI breast-cancer & 0.0914 & 0.0944 & 0.0934 & \underline{\textbf{0.0785}} & \textbf{0.0873} & \textbf{0.0803} \\
    diabetes & 0.5170 & 0.5180 & 0.5232 & \textbf{0.5083} & \textbf{0.5127} & \underline{\textbf{0.5068}} \\
    News20 & 0.5130 & 0.5177 & 0.5203 & \underline{\textbf{0.5072}} & \textbf{0.5100} & \textbf{0.5075} \\
    UCI Adult & 0.3383 & 0.3386 & 0.3549 & 0.3538 & 0.3384 & \underline{\textbf{0.3382}} \\
    cifar10 & 0.6847 & 0.6861 & 0.7246 & 0.6851 & \textbf{0.6822} & \underline{\textbf{0.6819}} \\
    MNIST & 0.0245 & 0.0247 & \textbf{0.0239} & \underline{\textbf{0.0223}} & \textbf{0.0245} & \textbf{0.0231} \\
    real-sim & 0.2606 & 0.2668 & 0.2644 & \underline{\textbf{0.2605}} & 0.2607 & 0.2609 \\
    SVHN  & 0.6129 & \textbf{0.6128} & 0.6757 & 0.6328 & \underline{\textbf{0.6122}} & \textbf{0.6128} \\
    skin-nonskin & 0.3527 & \underline{\textbf{0.3526}} & 0.3529 & 0.4830 & 0.3713 & \textbf{0.3527} \\
    Criteo1\% & 0.4763 & 0.4768 & 0.4953 & 0.4786 & \underline{\textbf{0.4755}} & \textbf{0.4756} \\
    Covertype & 0.6936 & \textbf{0.6933} & \textbf{0.6907} & 0.7745 & \underline{\textbf{0.6872}} & \textbf{0.6876} \\
    Avazu-app &  0.3449 & 0.3449 &  0.3450 & 0.3576 & \textbf{0.3446} &  \underline{\textbf{0.3446}} \\ 
    Avazu-site &  0.4499 & 0.4499  & 0.4505 & 0.5736 & \textbf{0.4490} & \underline{\textbf{0.4486}}  \\
    Company & 0.1955 & 0.1956 & 0.1958 & 0.1964 & \underline{\textbf{0.1952}} & \textbf{0.1953} \\
    \bottomrule
    \end{tabular}
}
    \label{95results}
        \begin{tablenotes}
      \item[*] The UIDS is the abbreviation of our \textbf{U}nweighted \textbf{I}nfluence \textbf{D}ata \textbf{S}ubsampling methods. The Lin- and Sig- indicates the incorporated Linear and Sigmoid sampling functions, respectively.
      \end{tablenotes}
    \end{threeparttable}
\end{table*}

\begin{figure}[t]
\centering
\includegraphics[width=8.0cm]{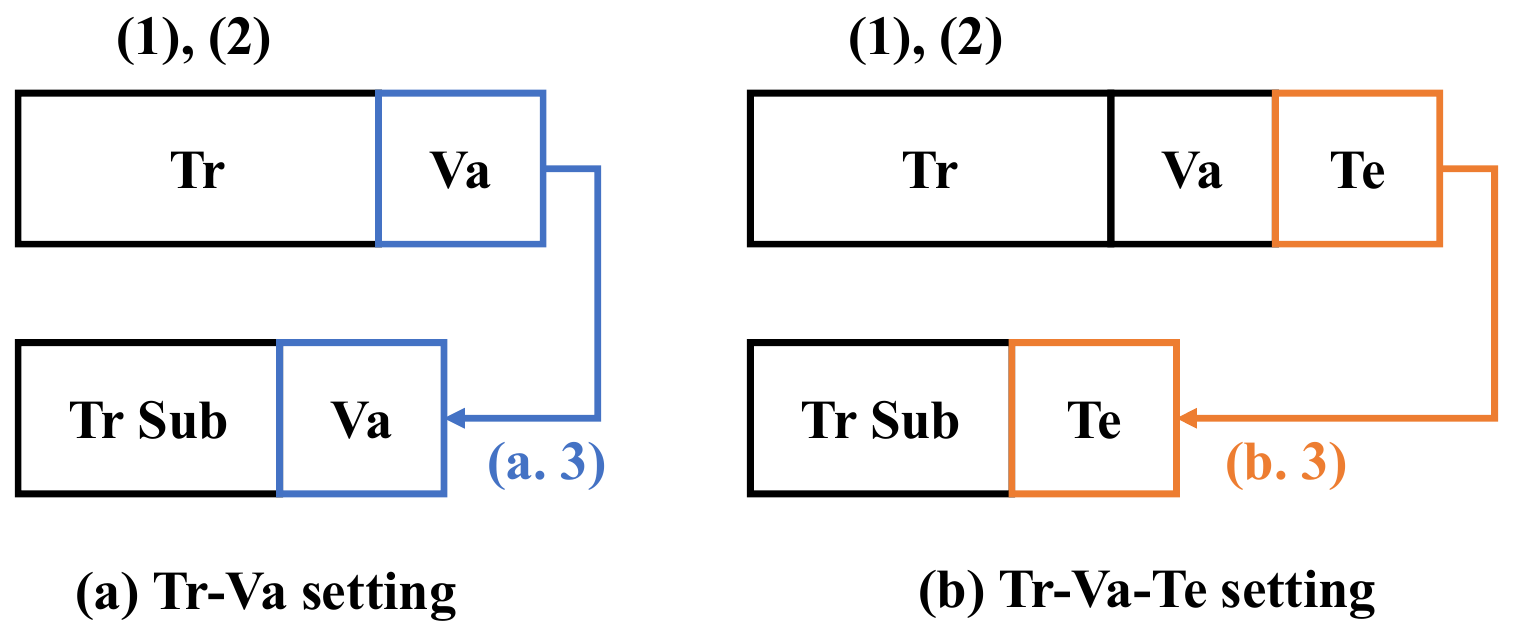}
\caption{Tr-Va v.s. Tr-Va-Te setting.}
\label{comparison}
\end{figure}

\subsection{Considered baselines}
We select the model trained on full data set and other three sampling approaches as the baselines for comparison: Optimal sampling \cite{ting2018optimal}, Data dropout \cite{wang2018data} and Random sampling. 

\begin{enumerate}[A.]
    \item \textbf{Optimal sampling.} It is a recent research in weighted subsampling that uses sampling probability proportional to $\|\psi_{\theta}\|$: $\pi_i = \max\{\alpha,\min\{1,\lambda\|\psi_\theta(z_i)\|\} \}$. This method aims at best approaching full set performance such that it cannot overtake full set theoretically.
    \item \textbf{Data dropout.} It is an unweighted subsampling approach which adopts a simple sampling strategy that is dropping out unfavorable samples whose $\phi_i>0$.
    \item \textbf{Random sampling.} It is simple random selection on Tr which means all samples share same probability. Theoretically, this strategy cannot win over full set as well.
\end{enumerate}

\subsection{Experiment protocols}
In our experiments, we use a Tr-Va-Te setting which is different from the Tr-Va setting as many previous work do (see the Fig. \ref{comparison}). Both settings proceed in three steps, and share the same first two steps: 1) training model $\hat{\theta}$ on the full Tr, predicting on the Va, then computing the IF; 2) getting sampling probability from the IF, doing sampling on Tr to get the subset, then acquiring the subset-model $\tilde{\theta}$. In the third step, we introduce an additional out-of-sample test set (Te) to test the $\tilde{\theta}$ (step (b.3)) rather than testing the $\tilde{\theta}$ on the Va (step (a.3)). The reason is that if we use the $\hat{\theta}$'s validation loss on the Va to guide the subsampling and then train the subset-model $\tilde{\theta}$, the testing result of $\tilde{\theta}$ on the Va cannot convince us of its generalization ability.

\begin{figure}	
	\centering
	\begin{subfigure}[t]{.2\textwidth}
		\centering
		\includegraphics[width=4.0cm]{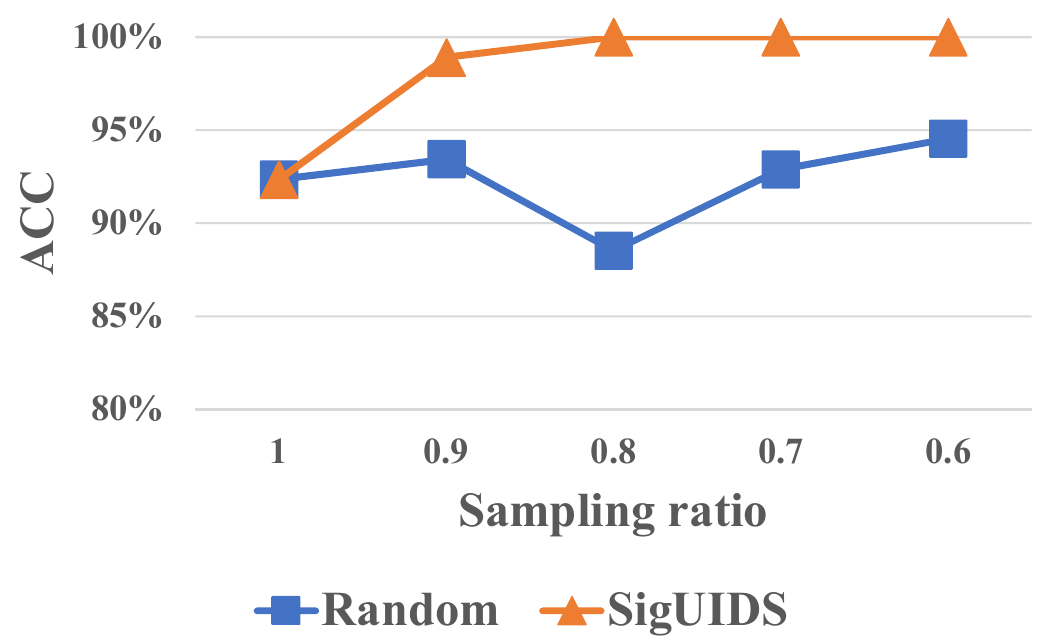}
		\caption{Breast cancer}	
	\end{subfigure}
	\quad
	\begin{subfigure}[t]{.2\textwidth}
		\centering
		\includegraphics[width=4.0cm]{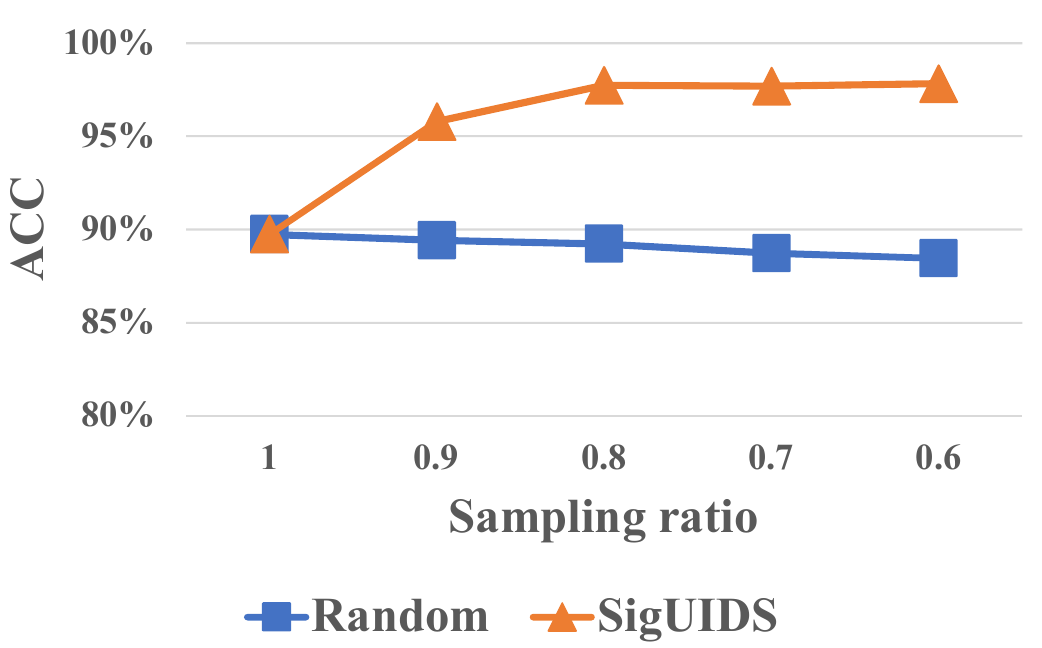}
		\caption{MNIST}
	\end{subfigure}
	\caption{Test ACC with noisy labels (40\% being flipped).}\label{fig:noisylabel}
\end{figure}

In fact, our framework is applicable for both convex and non-convex models, and we mainly focus on subsampling theory in this work. For implementation simplicity, we use logistic regression in all experiments. Besides, to ensure that our methods indeed achieve good performance in terms of the metrics they optimized for, i.e., we use the logistic loss for computing the influence function, such that the logloss is used as the metric in all experiments. More details about experimental settings can be found in appendix F.

\begin{figure}	[htbp]
	\centering
	\begin{subfigure}[t]{.2\textwidth}
		\centering
		\includegraphics[width=4.0cm]{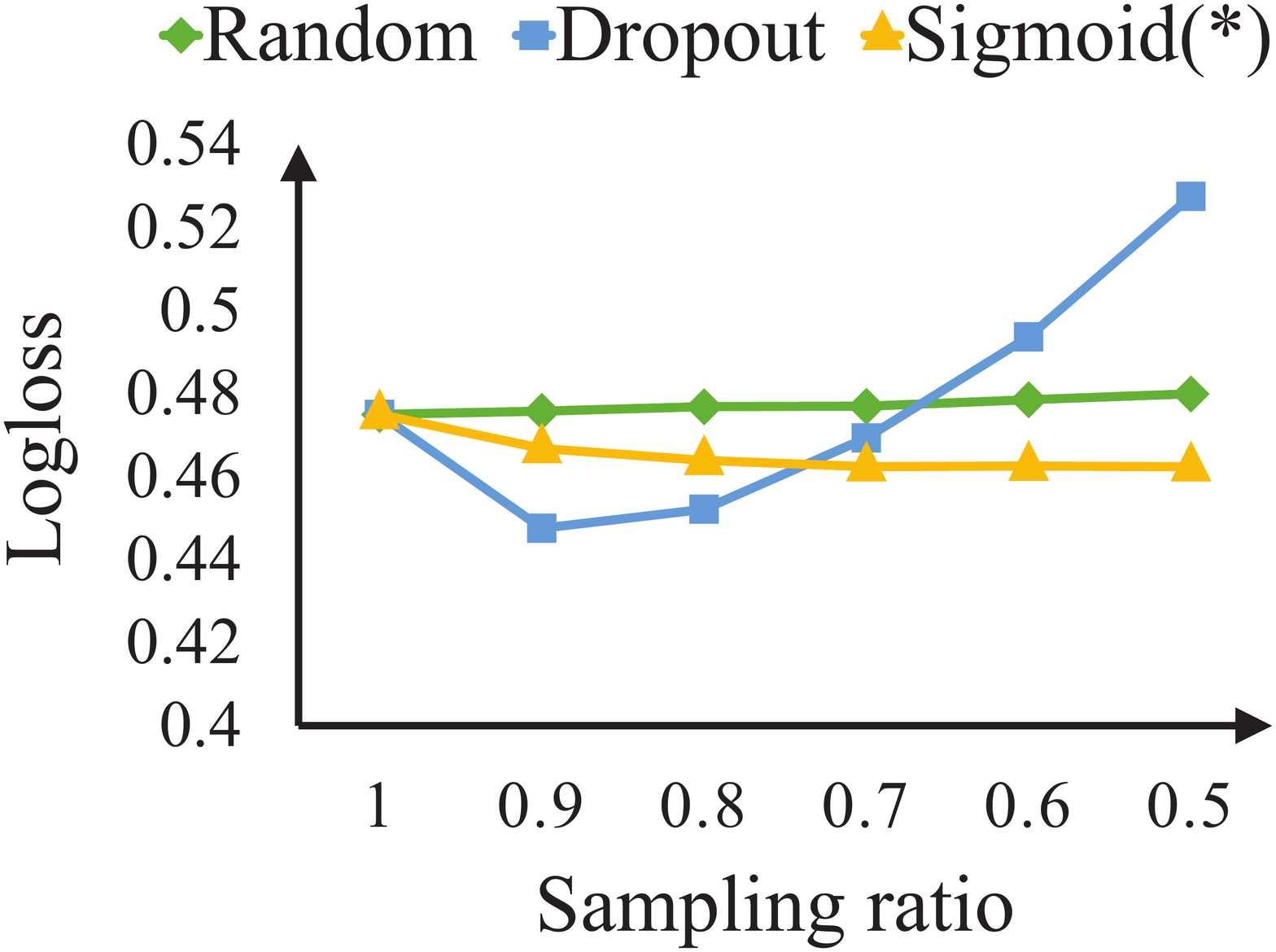}
		\caption{Validation logloss}	
	\end{subfigure}
	\quad
	\begin{subfigure}[t]{.2\textwidth}
		\centering
		\includegraphics[width=4.0cm]{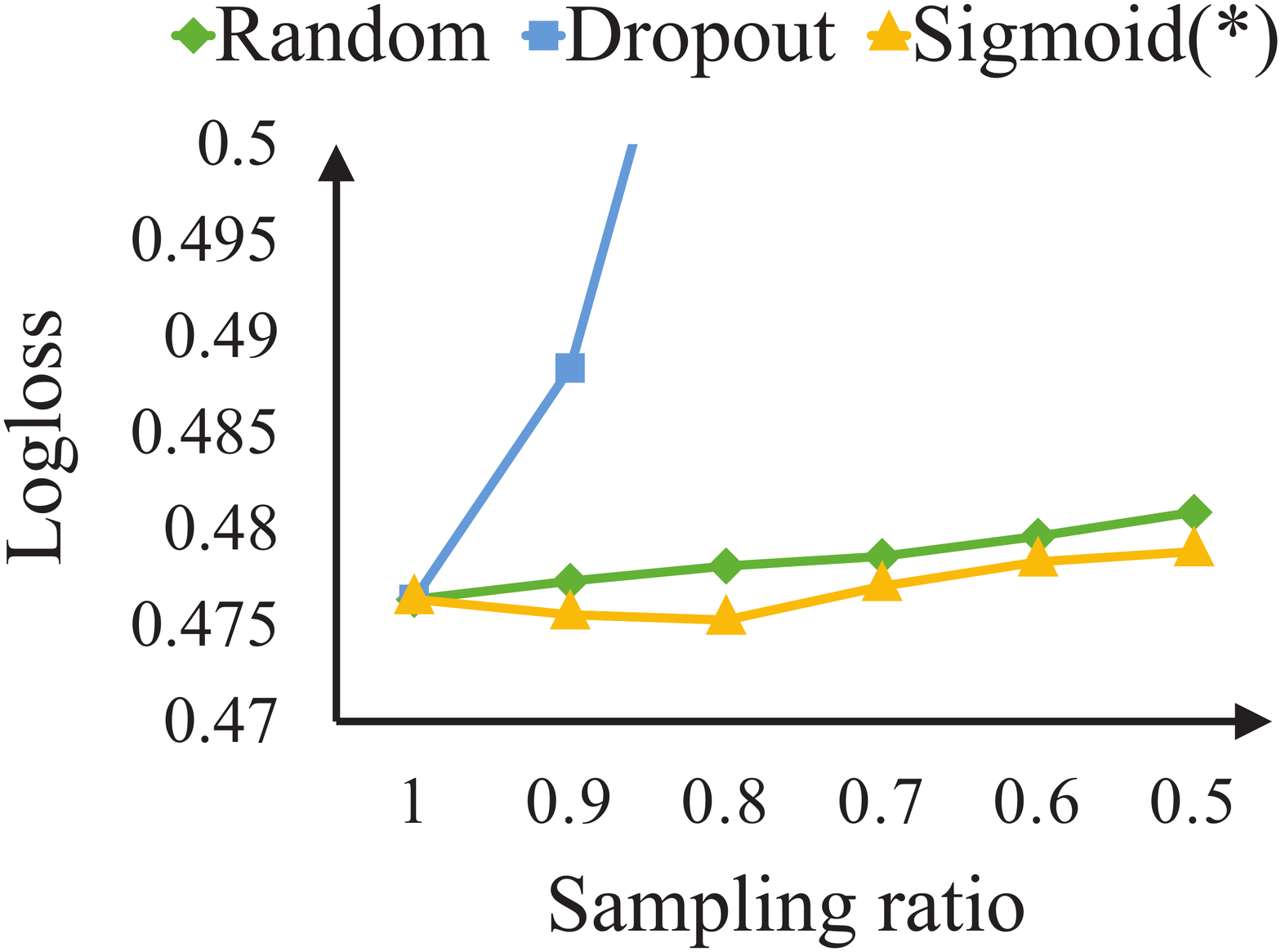}
		\caption{Test logloss}
	\end{subfigure}
	\caption{Comparison on logloss between Va set and Te set.}\label{fig:trvate}
\end{figure}

\begin{figure}[t]
\centering
\includegraphics[width=4.0cm]{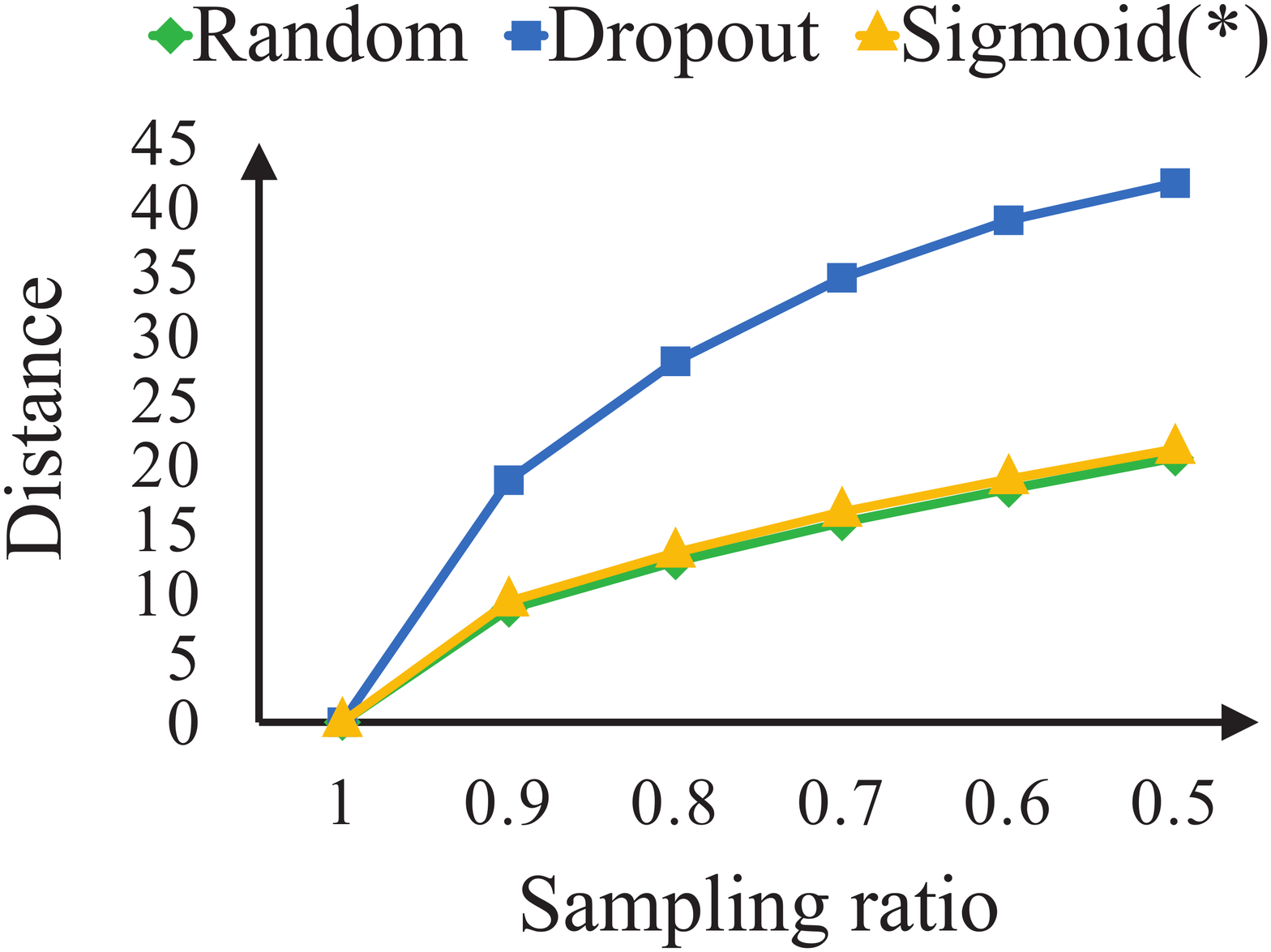}
\caption{The confidence degree $\|\tilde{\theta}-\hat{\theta}\|_2^2$.}
\label{diffparam}
\end{figure}

\subsection{Experiment observations}
\subsubsection{Result 1: Effectiveness.}
The experimental results are shown in the Table \ref{95results} where sampling ratio is set as $95\%$ for all sampling methods. The average test logloss \footnote{repeat 10 times sampling from the same Tr set} values are listed in each column. The bold letters indicate logloss less than the full-set-model, and the underlying ones are the best across the row. It can be seen that our Sig- and Lin-UIDS overhaul the full-set-model on most of data sets, while Dropout often fails. Besides, due to the high variance incurred by weight term, the OptLR method severely suffers. In a nutshell, the Sig-UIDS performs the best on 5 of 14 datasets, and both Lin-UIDS and Dropout achieves 4. That means over-confidence sometimes is beneficial on those homogeneous data sets, e.g. the MNIST, but the Dropout fails on all relatively large-scale and heterogeneous data sets. The probabilistic sampling methods have universal superiority over others, since it keeps robustness on a set of distributions $\mathcal{Q}$, rather than a specific $Q^{\prime}$ (the Va set).

Our unweighted method can downweight the bad cases which cause high test loss to the our model, which is an important reason of its ability to improve result with less data. To show the performance of our methods in noisy label situation, we perform addtional experiments with some training labels being flipped. The results show the enlarging superiority of our subsampling methods in Fig. \ref{fig:noisylabel}.
\begin{table}[t]
  \centering
  \setlength{\tabcolsep}{1.5mm}{
  \begin{threeparttable}[b]
  \caption{Time costs of computing the influence function on whole training set.}
    \begin{tabular}{lrrr}
    \toprule
    Dataset & \multicolumn{1}{l}{\# samples} & \multicolumn{1}{l}{\# features} & \multicolumn{1}{l}{Time cost (sec)} \\
    \midrule
    diabetes\tnote{a} &         768  &          8  & 0.03 \\
    news20\tnote{a} &      19,954  &   1,355,192  & 2.77 \\
    cifar10\tnote{a} &      60,000  &       3,072  & 5.85 \\
    criteo1\%\tnote{a}&     456,674  &   1,000,000  & 38.85 \\
    Avazu-app\tnote{b} &  14,596,137  &   1,000,000  &  166.56 \\
    Avazu-site\tnote{b} &  25,832,830  &   1,000,000  &  310.66 \\
    \bottomrule
    \end{tabular}%
        \begin{tablenotes}
      \item[a] Run on the Intel i7-6600U CPU @2.60GHz.
      \item[b] Run on the Intel Xeon CPU E5-2670 v3 @2.30GHz.
      \end{tablenotes}
    \label{runningtime}%
    \end{threeparttable}}
\end{table}%

\subsubsection{Result 2: Robustness.}

In Fig. \ref{fig:trvate}, we can see that the Dropout method performs very well on Va set, however, it fails in out-of-sample test. To illustrate how the proposed surrogate metric for confidence degree works, we set sampling ratio from large to small, then observe how the surrogate metric $\Gamma(\vec{\phi})$ changes. See in Fig. \ref{diffparam}, the Dropout causes large shift, while our Sig-UIDS has as small shift as the Random sampling. This phenomenon coincides with our Theorem \ref{lip} that the Lipschitz constant $\xi=\mathcal{O}(\frac{\sigma}n)$. With the presence of proper $\sigma$ in Sig-UIDS, the majority of $\pi$ are around $0.5$, which makes the sampling process more smooth.

\subsubsection{Result 3: Efficiency.}
The Table \ref{runningtime} shows summary of running time. For most of the data sets, our method can calculate the IF within one minute. With large and sparse data sets, our method can achieve computation within ten minutes, which is acceptable in practice.

\section{Conclusion \& Future Work}
In this work, we theoretically study the unweighted subsampling with IF, then propose a novel unweighted subsampling framework and design a family of probabilistic sampling methods. The experiments show that 1) different from the previous weighted methods, our unweighted method can acquire the subset-model that indeed wins over the full-set-model on a given test set; 2) it is crucial to evaluate the confidence degree over the empirical distribution for enhancing our subset-model $\tilde{\theta}$'s generalization ability.

Although our framework of the Unweighted Influence Data Subsampling (UIDS) method succeeds in improving model accuracy, there are still some interesting ideas remain to be explored. Since our framework is applicable both for convex and non-convex models, we can further testify its performance on those non-convex models, e.g. Deep Neural Networks. Another direction is to develop better approaches to deal with the over fitting issue, e.g. build a validation set selection scheme. Besides, we plan to implement our method in industry in the future.

\section{Acknowledgement}
The research of Shao-Lun Huang was funded by the Natural Science Foundation of China 61807021, Shenzhen Science and Technology Research and Development Funds (JCYJ20170818094022586), and Innovation and entrepreneurship project for overseas high-level talents of Shenzhen (KQJSCX20180327144037831). The authors would like to thank Professor Chih-Jen Lin's insight and advice for theory and writing of this work.

{\small
\bibliographystyle{aaai} 
\bibliography{references}  
}

\appendix
\setcounter{table}{0}
\setcounter{theorem}{0}
\setcounter{equation}{0}

\begin{appendix}

\section{Appendix A. Proof of Lemma 1}
\begin{lemma}\label{avgphi}
The expectation of the Influence function $\phi_i(\hat{\theta})$ over training distribution $P$ is always 0, which means:
\begin{equation}
    \mathbb{E}_{P}(\bm{\phi})\approx\frac1n\sum_{i=1}^n\phi_i(\hat{\theta})=0
\end{equation}
\end{lemma}

\begin{proof}
If $\pi_i=1$ for $\forall i=1,2,..,n$, then the expectation of $\bm{\phi}$ is simply the average on full training set. Based on the $\phi_i(\hat{\theta})$'s definition, we have
\begin{equation}
\begin{split}
        \frac1n \sum_{i=1}^n\phi_i(\hat{\theta}) &=\frac1n \sum_{j=1}^m\sum_{i=1}^n\phi(z_i,z_j,\hat{\theta}) \\
        &\approx\sum_{j=1}^m\sum_{i=1}^n\frac1{n\epsilon_i}(l_j(\hat{\theta}_{\epsilon})-l_j(\hat{\theta})) =0
\end{split}
\end{equation}
because $\hat{\theta}_\epsilon=\hat{\theta}$ in this scenario.
\end{proof}

\section{Appendix B. Proof of Lemma 2}
\begin{lemma}\label{nonneg}
The subset-model $\hat{\theta}_\epsilon$
performs not worse than the full-set-model $\hat{\theta}$ in terms of test risk $\mathcal{R}(Q^\prime)$ if $\bm{\epsilon}$ and $\bm{\phi}$ are negative correlated:
\begin{equation}
  Cov(\bm{\phi},\bm{\epsilon}) \leq 0
\end{equation}
\end{lemma}

\begin{proof}
Decomposing the expectation, we can get $\mathbb{E}(\bm{\phi}\times \bm{\epsilon})=\mathbb{E}(\bm{\phi})\times \mathbb{E}(\bm{\epsilon})+Cov(\bm{\phi},\bm{\epsilon})$. Based on the Lemma \ref{avgphi}, the $\mathbb{E}(\bm{\phi})=0$ such that $\mathbb{E}(\bm{\phi}\times \bm{\epsilon})=Cov(\bm{\phi},\bm{\epsilon}) \leq 0$, which means the subset-model's test risk $\mathcal{R}_{\hat{\theta}_\epsilon}(Q^\prime)$ is less or equal than the full-set-model's $\mathcal{R}_{\hat{\theta}}(Q^\prime)$.
\end{proof}

\section{Appendix C. Proof of Theorem 3}
\begin{theorem}\label{liprisk}
Let $\eta^{*}$ the optimal dual variable $\eta$ that achieves the infimum in the Eq. \eqref{infrisk}, and the perturbation function $\epsilon(\phi)$ has $\sigma$-bounded gradients. Then, the worst-case risk $\mathcal{R}_{\hat{\theta}_\epsilon}(\eta^{*})$ is a Lipschitz continuos function w.r.t. the IF vector $\vec{\phi}=(\phi_1,\phi_2,..,\phi_n)^{\top}$ where we have the Lipschitz constant $\xi=\mathcal{O}(\sigma\frac{\sqrt{2\delta+1}}{n})$, that is
\begin{equation*}
    \|\nabla_{\vec{\phi}}\mathcal{R}_{\hat{\theta}_\epsilon}(\eta^*)\| \leq \sigma\frac{\sqrt{2\delta+1}}{n} \times \sqrt{\sum_{i=1}^n \phi_i^2}
\end{equation*}
\end{theorem}

\begin{proof}
In order to measure the $\hat{\theta}_\epsilon$'s performance on an uncertainty set $\mathcal{Q}=\{Q \mid Q \ll P, D_f(Q \| P) \leq \delta, \delta \geq 0 \}$, it is common to define the worst-case risk as $\mathcal{R}_{\hat{\theta}_\epsilon}(Q) = \sup_{Q\in\mathcal{Q}} \{ \mathbb{E}_Q[l(\hat{\theta}_\epsilon;Z)]\}$. And its dual form is given as:
\begin{equation} \label{infrisk}
    \mathcal{R}_{\hat{\theta}_\epsilon}(\eta)=\inf_{\eta\in\mathbb{R}}\{ \sqrt{2\delta+1}\times\mathbb{E}_P[(l(\hat{\theta}_\epsilon;Z)-\eta)^2_+]^{\frac12}+\eta\}
\end{equation}
whose gradient on the vector $\vec{\phi}$ is a vector: 
\begin{equation}
    \nabla_{\vec{\phi}}\mathcal{R}_{\hat{\theta}_\epsilon}(\eta^*)=(\frac{\partial \mathcal{R}_{\hat{\theta}_\epsilon}(\eta^*)}{\partial\phi_1},\frac{\partial \mathcal{R}_{\hat{\theta}_\epsilon}(\eta^*)}{\partial\phi_2},...,\frac{\partial \mathcal{R}_{\hat{\theta}_\epsilon}(\eta^*)}{\partial\phi_n})^{\top}
\end{equation}
where $\eta^*$ helps the Eq.\eqref{infrisk} reach infimum.
With no loss of generality, take one element and analyze its bound:

\begin{align}
    \frac{\partial \mathcal{R}_{\hat{\theta}_\epsilon}(\eta^*)}{\partial\phi_i} &= \sqrt{2\delta+1}\times \frac{\partial \mathbb{E}_P[(l(\hat{\theta}_\epsilon;Z)-\eta^*)^2_+]^{\frac12}}{\partial \phi_i} \\
    &= \sqrt{2\delta+1} \times \frac{\partial [(l_i(\hat{\theta}_\epsilon)-\eta^*)_{+}^2]^{\frac12}}{\partial \phi_i} \\
    & \leq \frac{\sqrt{2\delta+1}}{n} \times \frac{\partial |l(\hat{\theta}_\epsilon;z_i)-\eta^*|}{\partial \phi_i} \\
    & \leq  \frac{\sqrt{2\delta+1}}{n} \times |\frac{\partial l(\hat{\theta}_\epsilon;z_i)}{\partial \epsilon_i}|\times |\frac{\partial \epsilon_i}{\partial \phi_i}| \\
    & \leq \frac{\sqrt{2\delta+1}}{n}\sigma \times |\phi_i|
\end{align}
Hence we can get the bound of the norm $\| \nabla_{\vec{\phi}} \mathcal{R}_{\hat{\theta}_\epsilon}(\eta^*)\| $ as:
\begin{align}
    \| \nabla_{\vec{\phi}} \mathcal{R}_{\hat{\theta}_\epsilon}(\eta^*)\| &= \sqrt{\sum_{i=1}^n(\frac{\partial \mathcal{R}_{\hat{\theta}_\epsilon}(\eta^*)}{\partial\phi_i})^2} \\
    &\leq \frac{\sqrt{2\delta+1}}n \sigma \times \sqrt{\sum_{i=1}^n\phi_i^2}
\end{align}
That means the change rate of worst-case risk $\mathcal{R}_{\hat{\theta}_\epsilon}(\eta^*)$ is aligned with the $\xi = \mathcal{O}(\sigma\frac{\sqrt{2\delta+1}}{n})$.
\end{proof}

\section{Appendix D. Proof of Theorem 4}
\begin{theorem}\label{lip}
Let the perturbation function $\epsilon(\phi)$ has $\sigma$-bounded gradient, and the $|\epsilon(\phi)|$ is bounded by $\tau\in\mathbb{R}^+$, that is $|\epsilon(\phi)|\leq\tau$. We have the parameter shift $\Gamma(\vec{\phi})=\|\hat{\theta}_\epsilon -\hat{\theta}\|^2$ is Lischitz continuous with its Lipschitz constant $\xi=\mathcal{O}(\sigma\tau)$. Specifically for $\tau=\frac1n$, we have $\xi = \mathcal{O}(\frac{\sigma}n)$.
\end{theorem}

\begin{proof}
Note that $\Gamma:\mathbb{R}^d\to\mathbb{R}$, its gradient on $\vec{\phi}$ is also a vector with $n$ dimensions:
\begin{equation}
    \nabla_{\vec{\phi}}\Gamma(\vec{\phi})=(\frac{\partial \Gamma(\vec{\phi})}{\partial\phi_1},\frac{\partial \Gamma(\vec{\phi})}{\partial\phi_2},...,\frac{\partial \Gamma(\vec{\phi})}{\partial\phi_n})^{\top}
\end{equation}

In fact, proving $\Gamma(\vec{\phi})$ is Lipschitz continuous is equivalent to proving $\|\nabla_{\vec{\phi}}\Gamma(\vec{\phi})\|$ is bounded. Let's select one arbitrary element from the vector $\nabla_{\vec{\phi}}\Gamma(\vec{\phi})$ and try to derive its bound:
\begin{align}
    |\frac{\partial \Gamma(\vec{\phi})}{\partial\phi_i}| &= |\frac{\partial \Gamma(\vec{\phi})}{\partial\hat{\theta}_\epsilon}\frac{\partial \hat{\theta}_\epsilon}{\partial\epsilon_i}\frac{\partial \epsilon_i}{\partial\phi_i}| \\
    & \approx |2(\hat{\theta}_\epsilon-\hat{\theta})^{\top}\psi_{\theta}(z_i)\frac{\partial \epsilon_i}{\partial\phi_i}| \label{approx1} \\
    & \leq 2\sigma|(\hat{\theta}_\epsilon-\hat{\theta})^{\top}\psi_{\theta}(z_i)| \label{ieq1}\\
    & \leq 2\sigma\|\hat{\theta}_\epsilon-\hat{\theta}\|\|\psi_\theta(z_i)\| \label{ieq2}
\end{align}
The first approximation Eq.\eqref{approx1} comes from the definition of Influence function on parameters since $\frac{\partial \theta_\epsilon}{\partial \epsilon_i} \approx\psi_\theta(z_i)$ when $\epsilon_i \to 0$. The first inequality Eq.\eqref{ieq1} holds since $\epsilon_i=\epsilon(\phi_i)$ as $\sigma$-bounded gradients. The second inequality Eq.\eqref{ieq2} comes from the Cauchy-Schwartz inequality.

Note that $|\epsilon(\phi)|$ is bounded, the $\|\hat{\theta}_\epsilon-\hat{\theta}\|$ must be bounded as well. Here we can make an approximation that $\hat{\theta}_\epsilon-\hat{\theta}=\sum_{i=1}^n \epsilon_i \psi_\theta(z_i)$ if each $\epsilon$ is small, such that
\begin{align}
\|\hat{\theta}_\epsilon-\hat{\theta}\| &= \|\sum_{i=1}^n\epsilon_i\psi_\theta(z_i)\| \\
& \leq \sum_{i=1}^n \|\epsilon_i\psi_\theta(z_i)\| \label{ieq3}\\
& \leq \tau\sum_{i=1}^n \|\psi_\theta(z_i)\| \label{ieq4}
\end{align}
The second inequality Eq.\eqref{ieq3} holds because $\epsilon(\phi)$ is bounded by $\tau \in \mathbb{R}^+$. Combine the Eq.\eqref{ieq2} and Eq.\eqref{ieq4}, we can derive that $|\frac{\partial \Gamma(\vec{\phi})}{\partial\phi_i}|$ is bounded, such that the $\|\nabla_{\vec{\phi}}\Gamma(\vec{\phi})\|$ is bounded:
\begin{align}
\|\nabla_{\vec{\phi}}\Gamma(\vec{\phi})\| &= \sqrt{\sum_{i=1}^n(\frac{\partial \Gamma(\vec{\phi})}{\partial \phi_i})^2} \\
& \leq \sqrt{\sum_{i=1}^n max(|\frac{\partial \Gamma(\vec{\phi})}{\partial \phi_i}|)} \\
&= 2\sigma\tau\sum_{i=1}^n\|\psi_\theta(z_i)\|\sqrt{\sum_{i=1}^n \|\psi_{\theta}(z_i)\|^2} \label{libconst}
\end{align}
Therefore, we can conclude that (see the Eq. \eqref{libconst}), it is easy to derive that the Lipschitz constant $\xi=\mathcal{O}(\sigma \tau)$. Specifically for $\tau=\frac1n$ (i.e. the $\epsilon \in [-\frac1n,0]$), we have $\xi=\mathcal{O}(\frac{\sigma}n)$.
\end{proof}

\section{Appendix E. Data Sets and Experimental Settings}
\begin{table}[t]
  \centering
  \begin{threeparttable}[t]
  \caption{Data sets statistics}
 \setlength{\tabcolsep}{2mm}{
    \begin{tabular}{lrrr}
    \toprule
    Dataset & \multicolumn{1}{l}{\# samples} & \multicolumn{1}{l}{\# features} & Domain \\
    \midrule
    UCI breast-cancer  &           683  &         10  & Medical \\
    diabetes &           768  &          8  & Medical \\
    news20 &        19,954  &  1,355,192  & Text \\
    UCI Adult &        32,561  &        123  & Society \\
    cifar10 &        60,000  &      3,072  & Image \\
    MNIST &        70,000  &        784  & Image \\
    real-sim &        72,309  &     20,958  & Physics \\
    SVHN  &        99,289  &      3,072  & Image \\
    skin non-skin &       245,057  &          3  & Image \\
    criteo1\% &       456,674  &  1,000,000  & CTR \\
    Covtype &       581,012  &         54  & Life \\
    Avazu-app &    14,596,137  &  1,000,000  & CTR \\
    Avazu-site &    25,832,830  &  1,000,000  & CTR \\
    \textbf{Company} & $\approx100$M & $\approx10$M & CTR \\
    \bottomrule
    \end{tabular}}
    \label{dataset}%
    \end{threeparttable}

\end{table}%

\subsection{Data set}
The data sets statistics can be found in Table \ref{dataset}, and several of them are processed specifically.
\subsubsection{MNIST, cifar10 and SVHN.}
They are all 10-classes image classification data sets while Logistic regression can only handle binary classification. On MNIST and SVHN, we select the number 1 and 7 as positive and negative classes; On cifar10, we do classification on cat and dog. For each image we convert all pixels to flattened feature values with all being scaled by $1/255$.
\subsubsection{Covertype.} It is a multi-class forest cover type classification dataset which is transformed to binary class and all features are scaled to $[0,1]$.
\subsubsection{News20.} This is a size-balanced two-class variant of the UCI 20 Newsgroup data set where the each class contains 10 classes and each example vector is normalized to unit length.
\subsubsection{Criteo1\%.} It is used in a CTR prediction competition held jointly by Kaggle and Criteo in 2014. The data used here is conducted feature engineering according to winning solution in this competition. We ramdomly sample $1\%$ examples from the original data set.
\subsubsection{Avazu-app and Avazu-site.} This data is used in a CTR prediction competition held jointly by Kaggle and Avazu in 2014. Here the data is generated according to winning solution where the data is split into two groups "app" and "site" for better performance.

\subsection{Experimental settings}
For logistic regression on both full set and subset, we select the regularization term $C=0.1$ for fair comparison. For the Optimal sampling methods, we set $\lambda=1/\max\{\|\psi_\theta(z_i)\|\}_{i=1}^n$ to scale the probability into $[0,1]$ and set $\alpha=0.01$ to prevent the $\frac1{\pi}$ from large variance following. For Data dropout method, we rank the samples by their IF and select the top ones; For Linear sampling function, we set $\alpha=1/\max\{|\phi_i|\}_{i=1}^n$ similar to the Optimal sampling and we randomly pick those unfavorable samples if samples with $\phi<0$ are not enough to reach the objective sampling ratio; for Sigmoid sampling, we set $\alpha\in\{0.1,1,5,10,50\}$.

For public data, we randomly pick up $30\%$ data from Tr as the Va for each data set. For the company data, with domain knowledge we use 7 days data as Tr, 1 day for Va and 1 day for Te. For all subsampling methods, the Tr, Va and Te maintain the same for fair comparison. Besides, to make the test logloss comparable among different subsampling methods, postive-negative sample ratio is kept invariant after subsampling for all methods, which avoids the test logloss being influenced by the shift of label ratio. 
\end{appendix}

\end{document}